\documentclass[journal]{IEEEtran}
\IEEEoverridecommandlockouts

\usepackage{cite}
\usepackage{amsmath,amssymb,amsfonts,amsthm}
\usepackage{algorithmic}
\usepackage{graphicx}
\usepackage{textcomp}
\usepackage{xcolor}
\usepackage{booktabs}
\usepackage{multirow}
\usepackage{array}
\usepackage{tikz}
\usepackage[colorlinks,linkcolor=red,anchorcolor=red,citecolor=blue]{hyperref}
\usepackage[lined,boxed,ruled]{algorithm2e}

\graphicspath{{./figs/}}

\newtheorem{assumption}{Assumption}
\newtheorem{lemma}{Lemma}
\newtheorem{proposition}{Proposition}
\newtheorem{theorem}{Theorem}
\newtheorem{remark}{Remark}

\def\BibTeX{{\rm B\kern-.05em{\sc i\kern-.025em b}\kern-.08em
    T\kern-.1667em\lower.7ex\hbox{E}\kern-.125emX}}

\begin{document}

\title{MMAO: Mechanism and Stability Analysis of Metabolic Closed-Loop Metaheuristics}

\author{Jinliang Xu$^{*}$ and Liping Ma%
\thanks{Jinliang Xu received the bachelor's and master's degrees from Beijing University of Posts and Telecommunications, Beijing, China, where he also received the Ph.D. degree in computer science in 2020. He is with the China Academy of Information and Communications Technology, Beijing, China. His current research interests include heuristic optimization, agent-based applications, and blockchain systems. His e-mail: jlxufly@gmail.com and xujinliang@caict.ac.cn.}%
\thanks{Liping Ma is with the Department of Disease Control and Prevention, The Seventh Medical Center of Chinese PLA General Hospital, Beijing, China; e-mail: lipingmaqzx@163.com.}}

\maketitle

\begin{abstract}
This paper studies the Metabolic Multi-Agent Optimizer (MMAO) at the framework level rather than at the implementation or benchmark level. The central question is whether the metabolic resource loop of private energy, communal budget, role drift, and lifecycle turnover has a framework-level interpretation beyond narrative metaphor. We introduce a generic MMAO state model that abstracts away domain-specific move operators while retaining the resource bookkeeping that defines the framework. Under mild bounded-gain and bounded-spending assumptions, we establish boundedness and nonnegativity properties for private energy, communal budget, role state, and active population size. We then characterize three endogenous behavioral regimes of the loop: contraction under sustained resource deficit, reinvestment under surplus communal accumulation, and search redistribution under heterogeneous marginal returns across agents or subgroups. The analysis is intentionally conservative. It does not claim global convergence of the full adaptive system, universal superiority over specialist optimizers, or a complete stationary characterization of the resulting process. Instead, it identifies which internal regulation properties are generic consequences of the loop and which remain implementation specific. A compact mechanism-validation package on representative continuous and discrete MMAO realizations provides supporting empirical evidence for this reading, but is not intended to replace a full benchmark study. The analysis provides a bounded, regenerative, resource-regulated interpretation of MMAO, while leaving implementation-level performance questions open.
\end{abstract}


\section{Introduction}
\IEEEPARstart{M}{etaheuristic} optimization is full of adaptive devices: dynamic populations, restart rules, archives, self-adaptive parameters, operator-selection schemes, and resource-allocation schedules. These mechanisms often work, but they are frequently introduced as separate modules, each justified locally and analyzed only partially. As a result, a framework can become operationally effective while remaining theoretically fragmented. It may be hard to say which observed behaviors follow from a coherent control principle and which are simply consequences of implementation choices.

The Metabolic Multi-Agent Optimizer (MMAO) was proposed to organize heterogeneous search behavior through a private-public metabolic economy. Agents earn or lose private energy according to recent utility, donate part of successful gains to a communal pool, drift continuously between exploratory and exploitative roles, and are removed or regenerated when their search activity ceases to justify its cost. In this view, adaptation should emerge from a resource loop rather than from a stack of disconnected control heuristics.

Earlier MMAO studies established progressively stronger empirical evidence for that idea, spanning framework formulation, contraction, broader benchmarking, and domain derivation \cite{xu2026mmao,xu2026minimalmmao,xu2026largescale,xu2026mmaodyn,xu2026mmaocls}. However, these studies still leave a deeper question unresolved: \emph{what, in generic terms, does the metabolic loop itself guarantee or at least structurally encourage}?

This paper addresses that question at the framework level rather than at the benchmark level. We do not aim to provide one more large benchmark or one more application mapping. We also do not claim a full convergence theory for the complete adaptive system, a proof of optimality, or a proof that every MMAO instantiation will outperform specialist baselines. Those goals would be unrealistic for the present framework maturity. Instead, we seek a framework-level understanding of three core issues:
\begin{itemize}
    \item whether the key internal state variables of MMAO are naturally bounded under mild assumptions;
    \item whether population contraction, communal reinvestment, and search redistribution can be characterized as endogenous regime behaviors of the loop;
    \item which of MMAO's observed behaviors are generic consequences of the resource economy and which remain dependent on domain-specific implementation details.
\end{itemize}

These goals can be expressed more explicitly as four research questions:
\begin{itemize}
    \item \textbf{RQ1:} What is the right problem-independent state-space description of the MMAO loop?
    \item \textbf{RQ2:} Under what mild conditions are private energy, communal budget, role state, and active population size bounded?
    \item \textbf{RQ3:} How can contraction, reinvestment, and search redistribution be characterized as endogenous regimes of the loop?
    \item \textbf{RQ4:} Which observed MMAO behaviors should be interpreted as generic loop consequences rather than as domain-specific implementation artifacts?
\end{itemize}

We show that a broad class of MMAO-like update rules admits boundedness and regenerative interpretations. Private energy, communal budget, role state, and active population size can be kept in compact state regions by the same bookkeeping mechanisms that also produce turnover and reinvestment. This does not yet amount to a complete convergence theory for the full adaptive system. It does, however, support a narrower and more defensible claim: MMAO can be interpreted as a bounded closed-loop search system with identifiable operating regimes, rather than only as a metaphor-driven heuristic narrative.

The paper makes four contributions.
\begin{itemize}
    \item It defines a generic MMAO closed-loop state model that abstracts away domain-specific move operators while preserving the framework's defining resource logic.
    \item It establishes boundedness and turnover results for private energy, communal budget, role state, and active population size under mild assumptions on gain, donation, decay, and spending.
    \item It characterizes contraction, reinvestment, and search-redistribution regimes as endogenous loop behaviors rather than as separate externally attached modules.
    \item It proposes a mechanism-validation protocol for representative continuous and discrete MMAO realizations so that the theoretical claims can be tested empirically without turning the paper into another leaderboard study.
\end{itemize}

\section{Related Work}
The present paper lies at the intersection of four literature strands.

The first is parameter control and self-adaptation in evolutionary computation. Classical surveys have long emphasized that parameter adaptation is one of the main levers by which heuristics become robust, but also one of the hardest to organize cleanly \cite{eiben1999parameter,karafotias2015parameter}. Representative self-adaptive mechanisms have been studied in evolution strategies, differential evolution, and related paradigms \cite{Hansen2001,brest2006self,Stanovov2018,omeradzic2024self}. These works establish that bounded internal feedback can act as a real design principle rather than as a tuning afterthought. MMAO belongs to this broad family, but differs by treating energy, donation, and turnover as the primary feedback language rather than viewing mutation strength or crossover probability as the only internal state worth regulating.

The second strand is adaptive resource allocation in multi-population and cooperative search. In distributed differential evolution, cooperative particle swarm optimization, constrained multitask optimization, and dynamic on-demand multitask allocation, resource allocation has been shown to affect search efficiency materially \cite{li2022distributed,liu2022cooperative,dong2025effective,chu2024competitive,han2022multitask}. Related work on knowledge-guided evolution and modular configuration also shows that the management of historical success information can be elevated from a local operator tweak to a higher-level control layer \cite{jiang2023knowledge,cho2025configx}. These studies are especially relevant because they confirm that the allocation of evaluation budget is not just an engineering convenience. However, resource allocation is often embedded in specialized architectures or layered on top of existing optimizers. MMAO differs in that resource accounting is meant to be framework defining rather than secondary.

The third strand concerns dynamic population sizing, aging, archives, and turnover. Recent work has shown both empirically and theoretically that static population size can be suboptimal and that bounded resizing, aging, or archival memory may substantially alter algorithmic behavior \cite{Tanabe2014,antipov2024already,doerr2025speeding,liu2025less,doerr2023understanding,li2025scalable,bian2025archive,bian2025stochastic}. Dynamic adaptation and persistence have also been studied from broader lifelong or nonstationary perspectives \cite{leuzzi2025lifelong,signorelli2025perturbation}. This line is directly relevant to MMAO because its population is not merely initialized once and then preserved. Contraction and regeneration are part of the intended control logic.

The fourth strand concerns the broader criticism and analysis of modern metaphor-based heuristics. Benchmarking studies, search-behavior analyses, and recent theory papers increasingly warn that superficial novelty claims are easy to make when mechanisms are not clearly isolated and interpreted \cite{vermetten2024large,stripinis2024benchmarking,cenikj2025comparing,dang2025dominance,zheng2023runtime}. For a framework like MMAO, this creates a higher burden of proof. It is not enough to say that the method works on some tasks. One must also identify what is generic about the underlying loop and what can be justified independently of the metaphor.

We synthesize these lines of work to describe MMAO as a \emph{closed-loop resource-regulated metaheuristic architecture}. The theoretical task is then to make that claim precise.

\section{A Generic MMAO Closed-Loop Model}
\subsection{Three-Layer Abstraction}
The main modeling challenge is to analyze MMAO at the framework level without collapsing the theory into one narrow implementation. We therefore use a three-layer abstraction.
\begin{itemize}
    \item \textbf{Layer 1: closed-loop framework state.} This layer contains the generic variables that define MMAO as a framework: candidate states, private energy, communal budget, role state, active population size, and memory structures.
    \item \textbf{Layer 2: resource bookkeeping.} This layer specifies how gain is normalized, how private return is split between local retention and communal donation, how communal resources are spent, and how turnover is triggered.
    \item \textbf{Layer 3: domain instantiation.} This layer specifies how a particular domain generates candidate encodings, gain signals, and neighborhood moves. It affects performance strongly, but it is not what defines the metabolic loop itself.
\end{itemize}

The theoretical results in this paper are intended to live mainly at Layers 1 and 2. When a statement depends essentially on Layer 3, it should be read as an implementation theorem rather than as a framework theorem.

\subsection{Framework States}
We model MMAO at iteration $t$ through the state
\begin{equation}
    S_t=(X_t,E_t,B_t,\Phi_t,N_t,M_t),
\end{equation}
where:
\begin{itemize}
    \item $X_t=\{x_i(t)\}_{i=1}^{N_t}$ is the set of active candidate states;
    \item $E_t=(E_1(t),\ldots,E_{N_t}(t))$ is the vector of private energies;
    \item $B_t \ge 0$ is the communal budget;
    \item $\Phi_t=(\phi_1(t),\ldots,\phi_{N_t}(t))$ with $\phi_i(t)\in[0,1]$ is the vector of continuous role states;
    \item $N_t$ is the active population size;
    \item $M_t$ denotes memory, consensus, or archive-like structures maintained by the loop.
\end{itemize}

This representation is intentionally generic. The domain-specific meaning of $x_i(t)$ may differ across continuous optimization, combinatorial optimization, dynamic optimization, or classification-oriented mixed search. By contrast, $E_t$, $B_t$, $\Phi_t$, and $N_t$ are treated as framework-level variables.

\subsection{Cross-Domain Instantiation Reading}
Table~\ref{tab:instantiation} clarifies how the same generic loop can be read across different MMAO domains. The mapping identifies shared control variables across variants with different domain-specific behavior.

\begin{table*}[t]
\caption{Cross-domain reading of the generic MMAO state model.}
\label{tab:instantiation}
\centering
\scriptsize
\resizebox{\textwidth}{!}{%
\begin{tabular}{p{2.1cm}p{3.0cm}p{3.0cm}p{2.9cm}p{3.0cm}}
\toprule
Generic object & Continuous optimization & Discrete/combinatorial optimization & Dynamic optimization & Mixed-space classification \\
\midrule
$x_i(t)$ & real-valued candidate vector & tour, subset, or assignment structure & current dynamic candidate in a changing landscape & feature mask plus hyperparameter configuration \\
$u_i(t)$ & normalized objective improvement & normalized cost or profit improvement & recovery-adjusted improvement after change & validation utility after compactness and overfitting penalties \\
$E_i(t)$ & local search capital & local combinatorial search capital & current usefulness under change pressure & current usefulness of mixed model-selection behavior \\
$B_t$ & shared intensification/diversification reserve & shared reserve for re-seeding or subgroup reinforcement & recovery-oriented reserve after change & communal reserve for compact mixed-space exploration \\
$\phi_i(t)$ & exploration--exploitation role state & broad move vs local edit preference & recovery exploration pressure & subset exploration vs hyperparameter refinement state \\
$N_t$ & active swarm or agent count & active constructive/improvement agent count & active recovery population size & active mixed-space search population \\
$M_t$ & archive, anchors, elite consensus & elite tours, patterns, or packing memories & refreshed memory under nonstationarity & elite mask consensus and mixed configuration memory \\
\bottomrule
\end{tabular}
}
\end{table*}

\subsection{Reward, Donation, and Spending}
Each agent generates a raw gain signal $g_i(t)$, which may depend on objective improvement, normalized reward, structural compactness, constraint satisfaction, or post-change recovery, depending on the domain. The loop only requires that this gain be converted into a bounded effective return $u_i(t)$:
\begin{equation}
    u_i(t)=\Psi(g_i(t),\xi_t),
\end{equation}
where $\xi_t$ is a normalization state and $\Psi$ is a bounded transformation.

The generic MMAO bookkeeping can then be written as
\begin{equation}
    E_i(t+1)=\Pi_E\left((1-\alpha_t)E_i(t)+u_i(t)-d_i(t)-c_i(t)\right),
\end{equation}
where $\alpha_t \in [0,1)$ is an energy-decay factor, $d_i(t)\ge 0$ is the donated quantity, $c_i(t)\ge 0$ is local spending, and $\Pi_E$ projects onto the admissible energy interval.

The communal budget evolves as
\begin{equation}
    B_{t+1}=\Pi_B\left(B_t+\sum_{i=1}^{N_t} d_i(t)-s_t\right),
\end{equation}
where $s_t \ge 0$ is communal spending on spawning, reinforcement, recovery, or subgroup redistribution, and $\Pi_B$ projects onto the nonnegative budget axis or onto a bounded budget interval if an explicit cap is imposed.

\subsection{Role Drift and Turnover}
Role states follow a bounded update
\begin{equation}
    \phi_i(t+1)=\Pi_{[0,1]}\left(\phi_i(t)+\Delta \phi_i(t)\right),
\end{equation}
where $\Delta \phi_i(t)$ depends on recent success, stability, or pressure indicators and may move the agent toward more exploratory or more exploitative behavior.

Turnover is modeled by a survival indicator
\begin{equation}
    z_i(t+1)=\mathbf{1}\{E_i(t+1)\ge \theta_t \ \text{or protected by policy}\},
\end{equation}
after which low-energy agents may be removed. New agents may then be instantiated if the communal budget and population policy allow it. The resulting population update has the form
\begin{equation}
    N_{t+1}=N_t-R_t+A_t,
\end{equation}
where $R_t$ is the number of removals and $A_t$ is the number of admissions or respawns.

\subsection{Why This Model Is the Right Abstraction}
This abstraction leaves out the exact move operator, local search kernel, encoding geometry, and domain-specific scoring details. That is intentional. Those details strongly affect performance, but they are not what makes MMAO a framework. The framework identity lies in the fact that adaptation is mediated through resource accumulation, donation, bounded role drift, and regenerative turnover. If a theoretical statement depends critically on a very specific mutation operator, it is a theorem about one MMAO implementation, not about the loop itself.

\section{Assumptions and Analysis Scope}
The analysis uses a mild assumption layer. The assumptions are not meant to hold only for one codebase; they are intended to capture what a mature MMAO implementation should enforce by design.

\begin{assumption}[Bounded effective gain]
There exists $U_{\max}>0$ such that for all $i$ and $t$,
\[
|u_i(t)| \le U_{\max}.
\]
\end{assumption}

\begin{assumption}[Bounded donation and spending]
There exist $D_{\max}, C_{\max}, S_{\max}>0$ such that for all $i$ and $t$,
\[
0 \le d_i(t) \le D_{\max}, \quad 0 \le c_i(t) \le C_{\max}, \quad 0 \le s_t \le S_{\max}.
\]
\end{assumption}

\begin{assumption}[Protected population interval]
There exist integers $1 \le N_{\min}\le N_{\max}<\infty$ such that the implementation enforces
\[
N_{\min} \le N_t \le N_{\max}
\]
through bounded admissions and removals.
\end{assumption}

\begin{assumption}[Projected role state]
All role-state updates are projected to $[0,1]$.
\end{assumption}

\begin{remark}
These assumptions are already close to actual MMAO practice. Mature implementations normally normalize gain, cap donation and spending implicitly through bounded state arithmetic, restrict role variables to compact intervals, and enforce explicit or implicit population envelopes.
\end{remark}

\subsection{Analytical Roadmap}
The assumptions support the rest of the paper in a staged way. Assumptions 1--4 first imply compactness of the internal state region. Compactness then makes it meaningful to discuss regenerative turnover rather than uncontrolled extinction or divergence. Once boundedness and regeneration are in place, regime analysis becomes interpretable: contraction can be understood as negative drift within a bounded state space, reinvestment as surplus-triggered spending within the same bounded space, and redistribution as asymmetric future resource allocation among subgroups with different marginal returns. This staged logic is deliberate. We first show that the loop is a well-posed internal control system, and only then ask what recurrent behaviors that system induces.

\section{Boundedness and Stability Results}
\subsection{State Nonnegativity and Compactness}
\begin{proposition}[Basic state boundedness]
Under Assumptions 1--4, if $E_i(0)\in[0,E_{\max}]$, $B_0\in[0,B_{\max}]$, $\phi_i(0)\in[0,1]$, and $N_0\in[N_{\min},N_{\max}]$, then for all $t$,
\begin{align*}
E_i(t)&\in[0,\bar E], \\
B_t&\in[0,\bar B], \\
\phi_i(t)&\in[0,1], \\
N_t&\in[N_{\min},N_{\max}],
\end{align*}
for some finite constants $\bar E,\bar B$ determined by the gain, decay, spending, and population-envelope parameters.
\end{proposition}

\begin{proof}
The role and population bounds are immediate from projection and envelope enforcement. For energy, the update is affine in $E_i(t)$ with bounded input and projection. Hence
\[
E_i(t+1)\le (1-\alpha_t)E_i(t)+U_{\max},
\]
up to bounded additive constants absorbed into $U_{\max}$. Iterating yields a finite uniform bound whenever $\inf_t \alpha_t>0$ or an explicit projection is present. The communal budget follows the same argument with bounded inflow and outflow plus projection.
\hfill$\square$

A more explicit recursion bound is given in Appendix~A.
\end{proof}

\begin{remark}
The proposition shows that the bookkeeping operations keep MMAO within a compact internal state region under the stated assumptions.
\end{remark}

\subsection{Communal Budget Stability}
\begin{theorem}[Dissipative communal-budget bound]
Suppose there exists $\varepsilon>0$ such that whenever $B_t$ exceeds a threshold $B^\dagger$, expected communal spending dominates expected net communal inflow:
\[
\mathbb{E}\left[s_t-\sum_{i=1}^{N_t} d_i(t)\mid B_t\right]\ge \varepsilon \qquad \text{for } B_t\ge B^\dagger.
\]
Then the communal budget process is mean-reverting above $B^\dagger$ and admits a finite expectation bound
\[
\sup_t \mathbb{E}[B_t] < \infty.
\]
\end{theorem}

\begin{proof}[Proof sketch]
For $B_t\ge B^\dagger$, the conditional drift is negative:
\[
\mathbb{E}[B_{t+1}-B_t\mid B_t]\le -\varepsilon,
\]
up to the projection at zero. Standard Foster-Lyapunov drift reasoning for nonnegative processes then yields bounded first moments and recurrence toward the lower-budget region.

Appendix~A makes the drift argument more explicit by isolating a finite small set and summing the supermartingale inequality over stopping intervals.
\end{proof}

\begin{remark}
The theorem formalizes an intuitive MMAO requirement: communal accumulation is useful only if surplus budget eventually triggers more spending, more admissions, or more reinforcement. Communal resources influence search through spending and reinvestment. MMAO becomes a closed loop only when large budget induces stronger reinvestment pressure.
\end{remark}

\subsection{Turnover and Regenerative Population Stability}
\begin{theorem}[Bounded regenerative population process]
Assume admissions are allowed whenever $N_t<N_{\min}$ or when communal surplus and policy conditions permit bounded expansion, while removals occur whenever a nonempty subset of agents remains below a survival threshold sufficiently long. Then the population process is regenerative within $[N_{\min},N_{\max}]$: it cannot diverge to infinity, cannot remain identically extinct, and repeatedly re-enters protected interior states.
\end{theorem}

\begin{proof}[Proof sketch]
Upper boundedness follows from the population envelope. Lower boundedness follows because once $N_t<N_{\min}$, admissions force return to at least $N_{\min}$. Regeneration follows because turnover removes stale agents while admissions inject fresh states; hence the system repeatedly returns to admissible configurations with newly initialized or reallocated agents.

Appendix~A states the recurrent refresh condition more explicitly and shows how it yields repeated returns to an interior refreshed set.
\end{proof}

\begin{remark}
This theorem gives MMAO a more precise interpretation than ``dynamic population.'' The population is neither free-growing nor strictly fixed. It is a bounded regenerative process whose size reflects resource conditions.
\end{remark}

\section{Regime Analysis of the Metabolic Loop}
\subsection{Contraction Regime}
Contraction occurs when expected local return remains too weak to justify the maintenance cost of the active population. Define the net private drift of agent $i$ as
\begin{equation}
    \delta_i(t)=\mathbb{E}[u_i(t)-d_i(t)-c_i(t)-\alpha_t E_i(t)\mid S_t].
\end{equation}
If a sufficiently large fraction of agents satisfies $\delta_i(t)<0$ for a sustained interval, then private energies decay, removals rise, and admissions are limited unless communal support offsets the deficit.

\begin{proposition}[Contraction under sustained deficit]
Suppose there exists an interval $\mathcal{T}$ and constants $\gamma,\beta>0$ such that for at least a fraction $\beta$ of active agents,
\[
\delta_i(t)\le -\gamma \qquad \forall t\in \mathcal{T},
\]
and communal reinvestment remains below the replacement requirement. Then the expected active population decreases over $\mathcal{T}$ until either the deficit weakens or the population reaches the protected floor.
\end{proposition}

\begin{proof}[Proof sketch]
Negative drift reduces the energies of a positive fraction of agents. This raises removal probability. If admissions do not compensate, the population update $N_{t+1}=N_t-R_t+A_t$ acquires negative drift. The floor $N_{\min}$ prevents collapse beyond the protected minimum.
\end{proof}

\subsection{Reinvestment Regime}
Reinvestment occurs when recent success and communal accumulation jointly indicate that the loop can afford to amplify search. This may appear as spawning, increased scout activity, stronger diversity maintenance, or reinforcement around high-yield regions.

\begin{proposition}[Reinvestment under communal surplus]
Suppose communal drift is positive over a recent window and $B_t\ge B^\dagger$ while average effective return remains nonnegative. If the spending policy is monotone in communal surplus, then the expected admission or reinforcement activity is nondecreasing in $B_t$, and the loop enters a reinvestment regime characterized by positive search-capital redistribution.
\end{proposition}

\begin{proof}[Proof sketch]
By policy monotonicity, higher surplus induces higher expected communal spending on bounded reinvestment actions. Since the budget is not allowed to remain idle indefinitely, surplus translates into admissions, intensified exploration, or other funded search actions.
\end{proof}

\subsection{Search Redistribution Regime}
The most distinctive MMAO behavior is not merely that it grows or shrinks. It is that it redistributes search effort toward agents, roles, or subgroups with higher recent marginal returns. This redistribution may occur via direct communal support, survival asymmetry, or role drift.

Let $\mu_k(t)$ denote the expected marginal return of subgroup $k$ at time $t$, where a subgroup may be defined by role range, neighborhood type, archive ancestry, or explicit population partition.

\begin{proposition}[Endogenous search redistribution]
If communal spending, survival protection, or role updates are monotone in recent marginal return estimates, then differences in $\mu_k(t)$ induce systematic redistribution of search effort. In particular, subgroups with larger positive marginal return attract larger expected future resource share unless boundedness constraints saturate the allocation.
\end{proposition}

\begin{proof}[Proof sketch]
Any monotone policy linking gain to donation retention, survival, role privilege, or communal spending creates a positive feedback path between marginal return and future activity share. The boundedness conditions prevent unlimited concentration, but not directional redistribution.
\end{proof}

\subsection{A Phase Interpretation}
The three regimes can be summarized in a two-signal phase view with axes:
\begin{itemize}
    \item recent normalized return or success level;
    \item communal surplus level.
\end{itemize}
Low return and low surplus yield contraction or survival mode. Moderate return and low-to-moderate surplus yield conservative maintenance. High return with high surplus yields reinvestment. High heterogeneity of subgroup returns with nontrivial surplus yields redistribution-dominant behavior. This phase view describes MMAO as a resource-regulated regime-switching system.

\begin{figure}[t]
\centering
\resizebox{\columnwidth}{!}{%
\begin{tikzpicture}[x=0.9cm,y=0.9cm,>=latex]
    \draw[->,thick] (0,0) -- (5.8,0) node[below] {Recent normalized return};
    \draw[->,thick] (0,0) -- (0,4.9) node[left] {Communal surplus};
    \draw[dashed] (2.8,0) -- (2.8,4.7);
    \draw[dashed] (0,2.35) -- (5.6,2.35);

    \node[align=center] at (1.4,1.15) {\small Contraction\\\small or survival};
    \node[align=center] at (4.2,1.15) {\small Conservative\\\small maintenance};
    \node[align=center] at (1.4,3.55) {\small Buffered\\\small reserve};
    \node[align=center] at (4.2,3.6) {\small Reinvestment\\\small and redistribution};

    \draw[rounded corners=2pt, thick, blue!60] (3.15,2.75) rectangle (5.25,4.35);
    \node[blue!60!black, align=center] at (4.2,2.5) {\small most distinctive\\\small MMAO regime};
\end{tikzpicture}%
}
\caption{Conceptual phase interpretation of the MMAO loop. The framework is most distinctive in the upper-right regime, where positive recent return and nontrivial communal surplus support reinvestment and endogenous redistribution of search effort.}
\label{fig:phase}
\end{figure}
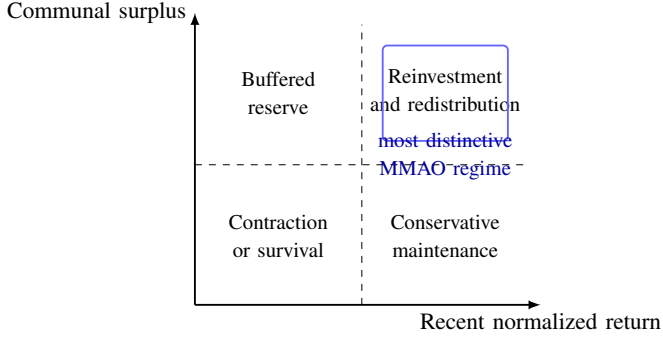

\section{What Is Generic and What Is Not}
The analysis so far supports a useful separation.

The following behaviors are generic consequences of the loop:
\begin{itemize}
    \item boundedness of internal resources under bounded-gain and bounded-spending policies;
    \item turnover-regulated regeneration of the active population;
    \item contraction under sustained net resource deficit;
    \item reinvestment under sustained communal surplus;
    \item directional search redistribution under heterogeneity of marginal returns.
\end{itemize}

By contrast, the following remain implementation specific:
\begin{itemize}
    \item the exact geometry of exploratory and exploitative moves;
    \item whether communal surplus is spent on spawning, local intensification, archive replay, or role-specific reinforcement;
    \item the exact responsiveness of role drift to recent success or pressure indicators;
    \item the finite-sample performance ranking against specialist baselines on any given benchmark family.
\end{itemize}

\begin{remark}
This distinction limits which empirical observations can be attributed to the generic framework.
\end{remark}

\section{Implications for Parameter-Light MMAO}
One of the long-term ambitions in the MMAO line is to move toward a more parameter-light, eventually quasi-parameter-free framework. The present analysis helps clarify what that ambition can and cannot mean.

The boundedness results suggest that several adaptive quantities do not need independently tuned schedules if they are already regulated by the loop. In particular, population size, role allocation, and communal spending intensity can be interpreted as \emph{state-derived variables} rather than as externally scripted time schedules. This does not make all constants disappear, but it changes their status. Thresholds and gains become structural coefficients in a closed-loop system, not separate hand-tuned modules layered on top of the search.

At the same time, the theory also explains why a fully parameter-free claim is still premature. The phase behavior depends on the sensitivity of donation, decay, and spending policies to observed gains. If these couplings are badly scaled, the same closed loop can become too conservative or too diffusive while remaining formally bounded. In other words, boundedness is necessary for mature self-regulation, but it is not sufficient for universally good calibration.

This suggests a principled roadmap. The right next step is not to keep adding external adaptive gadgets, but to derive more of the remaining constants from the same bookkeeping logic: scale adaptation from recent gain dispersion, spending intensity from communal budget pressure, and turnover thresholds from population-wide energy percentiles. This perspective is also consistent with broader attempts to move from hand-tuned modules toward learned or state-derived control in evolutionary computation \cite{jiang2023knowledge,cho2025configx}. The theory paper therefore does not solve the parameter-free problem outright, but it gives a cleaner criterion for future progress: every parameter that remains should either define the resource loop itself or be derivable from loop state.

\section{Mechanism-Validation Study Design}
\subsection{Purpose}
The experiment layer in this paper is intentionally small and mechanism centered (see GitHub repository \texttt{mmao}\footnote{\url{https://github.com/wolfbrother/mmao}} or PyPI package \texttt{mmao-opt}\footnote{\url{https://pypi.org/project/mmao-opt/}}). Its purpose is not to establish leaderboard performance, which has already been addressed more directly in broader benchmarking and representative domain-derivation studies \cite{xu2026largescale,xu2026mmaodyn}. Instead, the experiments should test whether the theoretical regime predictions actually appear in representative realizations of MMAO.

\subsection{Representative Realizations}
Two mature realizations are sufficient for this purpose:
\begin{itemize}
    \item a continuous MMAO instance on representative continuous landscapes such as sphere and Rastrigin;
    \item a discrete MMAO instance on representative combinatorial problems such as TSP or MKP.
\end{itemize}
These two cases are enough to test whether the resource-loop claims survive across different state encodings and gain-generation mechanisms.

\subsection{Research Questions for Mechanism Validation}
The mechanism-validation experiments should address four questions.
\begin{itemize}
    \item \textbf{MV1:} Do communal budget, private energy, role state, and active population size remain bounded in practice over long runs?
    \item \textbf{MV2:} Do low-yield regimes exhibit visible contraction and floor-protected survival?
    \item \textbf{MV3:} Do surplus-producing regimes exhibit stronger admissions or broader search reinvestment?
    \item \textbf{MV4:} Do ablations that remove communal sharing, turnover, or role drift suppress the predicted regime behaviors?
\end{itemize}

\subsection{Metrics}
The reporting layer should focus on mechanism variables rather than only objective values:
\begin{itemize}
    \item mean and quantile bands of communal budget $B_t$;
    \item mean private energy and energy dispersion;
    \item active population size $N_t$;
    \item mean role state and role-state dispersion;
    \item turnover count, admission count, and subgroup resource shares;
    \item occupancy ratio of contraction, maintenance, reinvestment, and redistribution regimes.
\end{itemize}

\subsection{Minimal Result Package}
The mechanism-validation study calls for a compact result package.
\begin{itemize}
    \item One boundedness table summarizing observed minima, maxima, and terminal values of $B_t$, mean $E_t$, mean $\phi_t$, and $N_t$ across the continuous and discrete realizations.
    \item One regime-occupancy table reporting the proportion of iterations assigned to contraction, maintenance, reinvestment, and redistribution modes under at least one easy and one hard instance.
    \item One trajectory figure for the continuous side and one for the discrete side, each plotting objective trace together with $B_t$, $N_t$, and mean $\phi_t$.
    \item One ablation figure or table showing that removing communal sharing, turnover, or role drift selectively weakens the corresponding predicted regime behavior.
\end{itemize}
This package is intentionally small. Its job is to validate the mechanism claims of the paper, not to re-run the entire empirical agenda of the larger benchmark study \cite{xu2026largescale}.

\subsection{Ablation Design}
The following endogenous ablations are especially informative:
\begin{itemize}
    \item \texttt{NoResourceSharing}: removes communal pooling while preserving local search dynamics;
    \item \texttt{FixedPopulation}: disables admissions and removals;
    \item \texttt{FixedRoleState}: removes role drift;
    \item \texttt{NoReinvestment}: allows communal accumulation but blocks surplus-triggered spending.
\end{itemize}

Each ablation suppresses one theoretical pathway. If the theory is correct, removing the pathway should selectively weaken the corresponding regime behavior rather than collapsing every aspect of performance indiscriminately.

\subsection{Implemented Mechanism-Validation Package}
To move beyond a purely proposed validation layer, this paper reuses mature mechanism-diagnostic assets from the broader benchmark study \cite{xu2026largescale} and extracts a compact cross-domain package centered on one representative continuous setting and one representative discrete setting. The continuous side uses the $10D$ CEC-style \texttt{F12017} slice together with the ablations \texttt{NoRoleDrift}, \texttt{NoEliteReinvest}, and \texttt{WeakSuccessFeedback}. The discrete side uses \texttt{eil51} together with the ablations \texttt{NoEdgeMemory}, \texttt{NoGuidedReinvest}, and \texttt{WeakSuccessFeedback}. This is not intended as a new benchmark campaign. Its only purpose is to test whether the theory-level state variables and regime differences are visible in real MMAO trajectories.

\begin{figure*}[t]
\centering
\includegraphics[width=\textwidth]{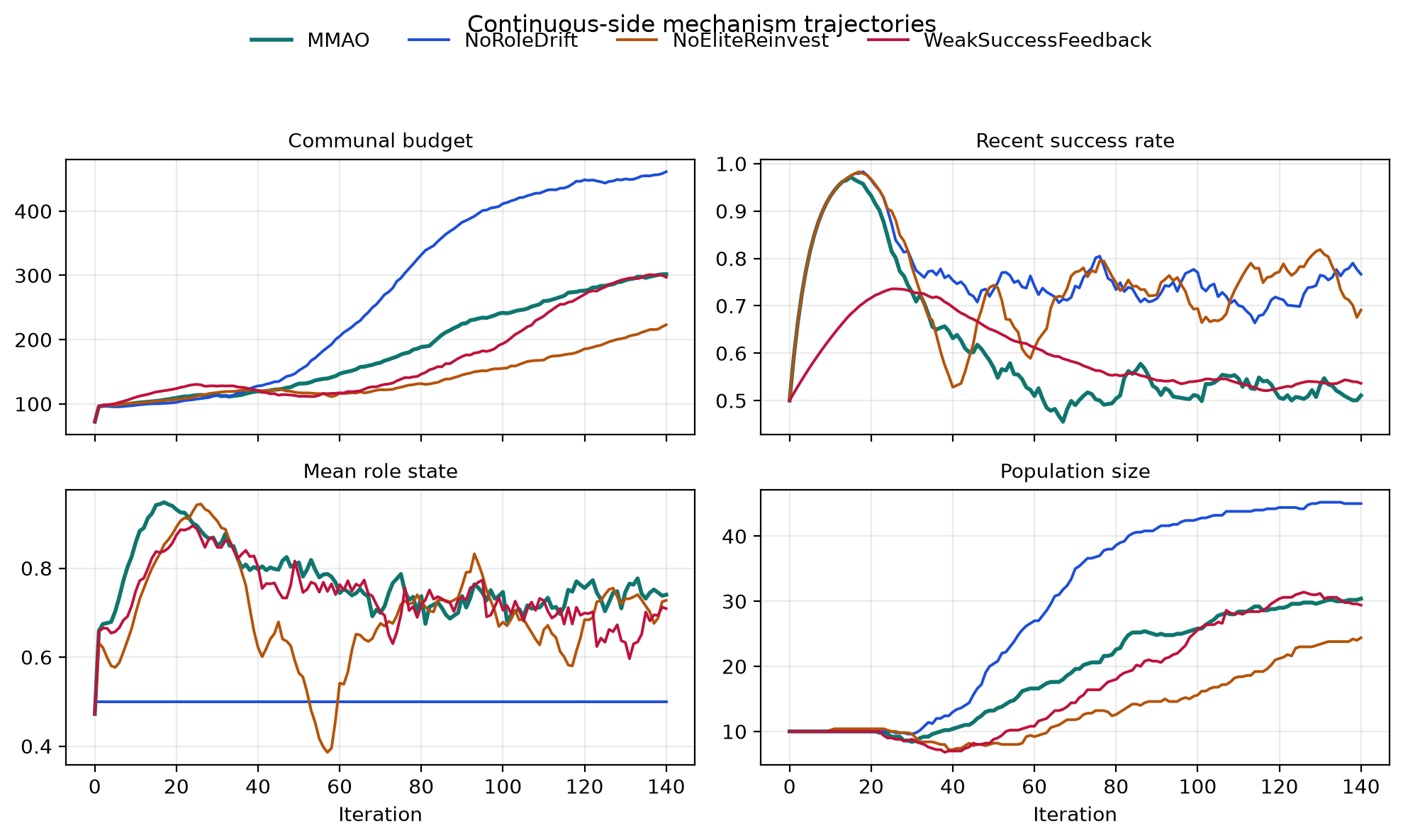}
\caption{Continuous-side mechanism trajectories for a representative MMAO slice and three endogenous ablations. The full MMAO realization exhibits bounded but growing communal budget, nontrivial role drift, and controlled population expansion. Removing role drift freezes the role-state channel and leads to much larger population and budget accumulation, whereas removing elite reinvestment suppresses communal growth and limits expansion.}
\label{fig:contmech}
\end{figure*}

\begin{figure*}[t]
\centering
\includegraphics[width=\textwidth]{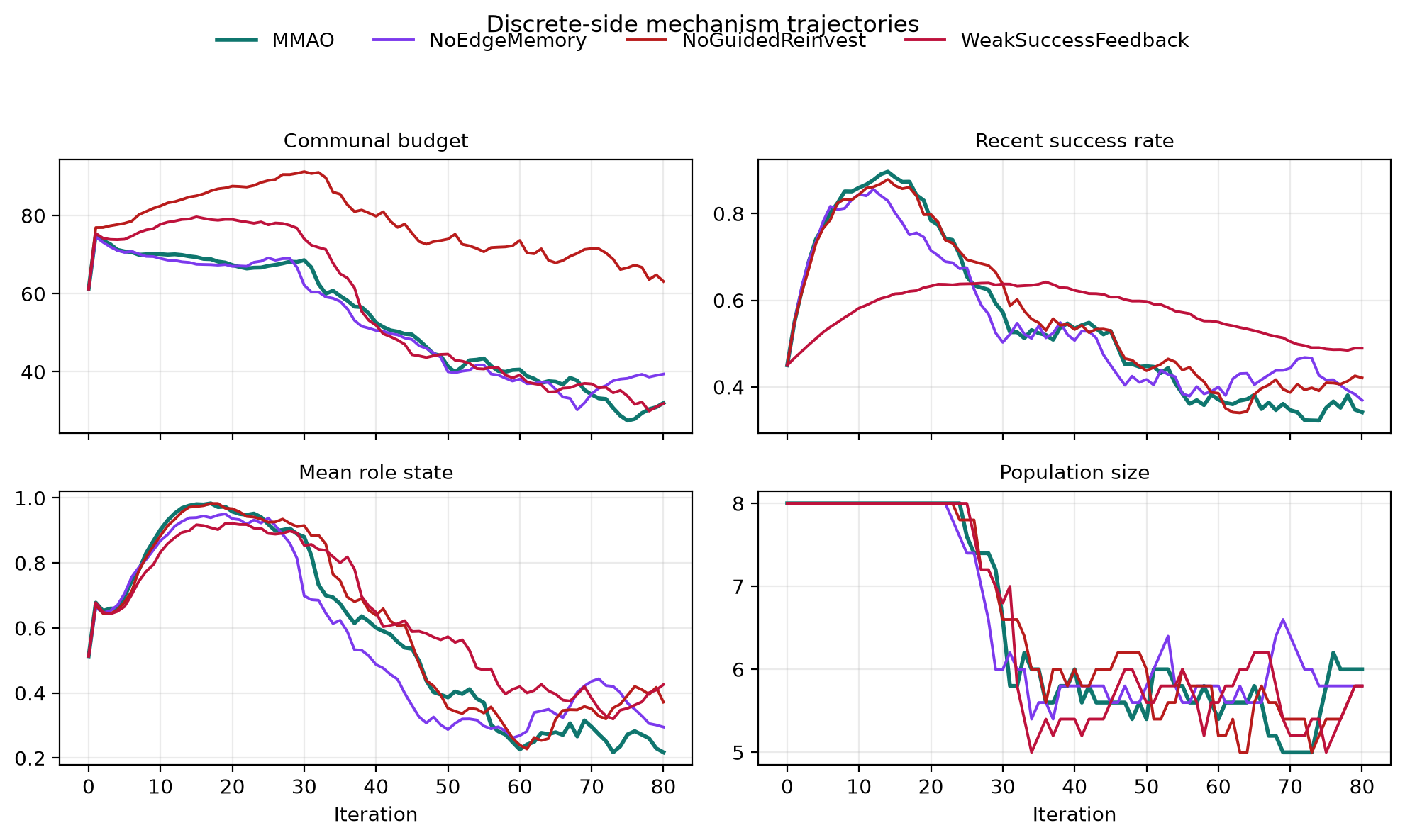}
\caption{Discrete-side mechanism trajectories for \texttt{eil51} and three endogenous ablations. The same state variables remain bounded but evolve differently from the continuous case: communal budget decays rather than accumulates, role state cools sharply, and the active population contracts toward a narrow band. This supports the paper's central distinction between framework-generic boundedness and domain-specific trajectory shape.}
\label{fig:discmech}
\end{figure*}

\begin{table*}[t]
\caption{Observed boundedness summary from the implemented mechanism-validation package.}
\label{tab:boundedness}
\centering
\scriptsize
\resizebox{\textwidth}{!}{%
\begin{tabular}{llcccccccccccc}
\toprule
Side & Method & $B_{\min}$ & $B_{\max}$ & $B_{\text{final}}$ & $s_{\min}$ & $s_{\max}$ & $s_{\text{final}}$ & $\phi_{\min}$ & $\phi_{\max}$ & $\phi_{\text{final}}$ & $N_{\min}$ & $N_{\max}$ & $N_{\text{final}}$ \\
\midrule
continuous & MMAO & 72.25 & 301.80 & 301.80 & 0.4554 & 0.9716 & 0.5108 & 0.4723 & 0.9481 & 0.7403 & 8.4 & 30.4 & 30.4 \\
continuous & NoRoleDrift & 72.25 & 461.35 & 461.35 & 0.5000 & 0.9830 & 0.7666 & 0.5000 & 0.5000 & 0.5000 & 9.6 & 45.2 & 45.0 \\
continuous & NoEliteReinvest & 72.25 & 223.09 & 223.09 & 0.5000 & 0.9829 & 0.6908 & 0.3861 & 0.9441 & 0.7275 & 7.2 & 24.4 & 24.4 \\
continuous & WeakSuccessFeedback & 72.25 & 300.94 & 297.06 & 0.5000 & 0.7357 & 0.5363 & 0.4723 & 0.8953 & 0.7091 & 6.8 & 31.4 & 29.4 \\
tsp & MMAO & 27.50 & 75.21 & 32.02 & 0.3223 & 0.8968 & 0.3416 & 0.2173 & 0.9833 & 0.2173 & 5.0 & 8.0 & 6.0 \\
tsp & NoEdgeMemory & 30.28 & 74.52 & 39.41 & 0.3693 & 0.8562 & 0.3693 & 0.2619 & 0.9507 & 0.2950 & 5.4 & 8.0 & 5.8 \\
tsp & NoGuidedReinvest & 61.20 & 91.24 & 63.14 & 0.3401 & 0.8787 & 0.4210 & 0.2277 & 0.9829 & 0.3718 & 5.0 & 8.0 & 5.8 \\
tsp & WeakSuccessFeedback & 29.94 & 79.66 & 31.87 & 0.4500 & 0.6422 & 0.4893 & 0.3200 & 0.9213 & 0.4259 & 5.0 & 8.0 & 5.8 \\
\bottomrule
\end{tabular}
}
\end{table*}

\begin{table}[t]
\caption{Ablation-effect summary used for mechanism reading. Lower is better for both columns.}
\label{tab:ablationeffects}
\centering
\scriptsize
\begin{tabular}{lcc}
\toprule
Variant & Cont. mean error & TSP mean gap (\%) \\
\midrule
MMAO & 3742.37 & 6.04 \\
NoEdgeMemory & -- & 5.88 \\
NoEliteReinvest & 80617.70 & -- \\
NoGuidedReinvest & -- & 5.67 \\
NoRoleDrift & 4044.08 & -- \\
WeakSuccessFeedback & 4133.13 & 5.47 \\
\bottomrule
\end{tabular}
\end{table}

\begin{table}[t]
\caption{Observed regime occupancy in the implemented mechanism-validation package.}
\label{tab:occupancy}
\centering
\scriptsize
\begin{tabular}{lcccc}
\toprule
Variant & Contr. & Maint. & Reserve & Reinvest. \\
\midrule
Cont. MMAO & 0.00 & 0.50 & 0.00 & 0.50 \\
Cont. NoRoleDrift & 0.00 & 0.49 & 0.00 & 0.51 \\
Cont. NoEliteReinvest & 0.00 & 0.52 & 0.00 & 0.48 \\
Cont. WeakSuccessFeedback & 0.00 & 0.62 & 0.00 & 0.38 \\
TSP MMAO & 0.00 & 0.32 & 0.00 & 0.68 \\
TSP NoEdgeMemory & 0.00 & 0.36 & 0.00 & 0.64 \\
TSP NoGuidedReinvest & 0.00 & 0.00 & 0.00 & 1.00 \\
TSP WeakSuccessFeedback & 0.00 & 0.38 & 0.00 & 0.62 \\
\bottomrule
\end{tabular}
\end{table}

\subsection{Mechanism Reading}
The implemented package supports three theory-relevant observations.

First, boundedness is visible in practice. Table~\ref{tab:boundedness} shows that communal budget, mean role state, and active population size remain in compact ranges on both continuous and discrete sides, even though the actual trajectory shapes are quite different. This is consistent with bounded internal regulation across realizations with different trajectory shapes.

Second, Figures~\ref{fig:contmech} and \ref{fig:discmech} reveal clear redistribution differences once ablations are introduced. On the continuous side, removing role drift collapses one adaptation channel entirely and coincides with much stronger communal accumulation and much larger population growth. Removing elite reinvestment sharply suppresses communal budget growth and leaves the algorithm in a visibly narrower expansion regime. On the discrete side, the most notable effect is that the full method contracts toward a tighter population band while still maintaining a comparatively broad role-state excursion, whereas the ablations either preserve excess budget or keep the role state artificially elevated.

Third, the ablation-effect summary in Table~\ref{tab:ablationeffects} supports a cautious causal reading. The continuous result is especially clear: removing elite reinvestment worsens the representative error by more than an order of magnitude, which is exactly the kind of selective mechanism failure the theory predicts. The discrete result is more mixed, which is also informative. It suggests that boundedness and redistribution are generic consequences of the loop, but that the performance value of any one pathway remains domain dependent.

Fourth, the regime-occupancy summary in Table~\ref{tab:occupancy} makes the same point from a different angle. The selected representative runs are not contraction-dominant failure cases; they spend most of their time in maintenance and reinvestment windows, which is exactly why they are useful for exposing surplus handling and redistribution logic. Even within that favorable setting, weakening success feedback or guided reinvestment systematically shifts occupancy toward the more conservative maintenance band. This is the practical regime signature that the theory predicts.

These observations are also consistent with three broader empirical motifs that recur across MMAO realizations. In static benchmark settings, stronger variants tend to benefit mainly by converting communal reserves into sharper late-stage redistribution rather than by simply maintaining larger populations. In dynamic settings, the same loop is most credible when it reallocates effort after disruption rather than merely preserving pre-change momentum. In mixed-space classification settings, the most visible effect is often compactness-oriented budget use rather than overwhelming predictive dominance. The theory presented here does not prove those empirical claims, but it does explain why such different outward behaviors can still arise from one common bounded controller.

\section{Discussion}
The theory developed here is intentionally partial. It does not establish global convergence of the full MMAO system, nor does it claim that every bounded resource-regulated heuristic is automatically effective. Several limitations remain.

First, the analysis treats the gain signal as bounded and normalized but does not derive the gain distribution from problem geometry. This is acceptable at the framework level, but it means that the present theorems explain \emph{how} the loop regulates itself more clearly than \emph{when} it will dominate specialist baselines.

Second, the population process is analyzed as bounded and regenerative, but not through a full stationary-distribution characterization. Such a characterization may be possible for simplified MMAO variants, but would likely obscure the main conceptual result of the present paper.

Third, role drift is modeled abstractly as a projected bounded state. This is enough to establish compactness and redistribution logic, but not enough to derive exact exploration-exploitation schedules. That level of detail belongs to implementation papers.

Fourth, the mechanism-validation section now includes real cross-domain trajectory figures and boundedness summaries, but it is still intentionally compact. The present package is enough to validate the theoretical reading, not enough to replace the broader empirical role of the larger benchmark study \cite{xu2026largescale}.

The analysis gives MMAO a state-space interpretation in terms of boundedness, regeneration, and operating regimes.

The paper also clarifies an important division of labor inside MMAO research. Implementation papers should determine which operators, memories, and proposal kernels are effective in a given domain. This theory paper asks a different question: which internal behaviors are generic consequences of the closed loop itself, regardless of whether the surrounding realization is continuous, discrete, dynamic, or mixed-space. By making that boundary explicit, the paper helps later empirical work make stronger claims without over-attributing every observed gain to theory.

\section{Conclusion}
This paper studied the framework-level properties of MMAO's metabolic resource loop independently of a particular implementation. A broad class of MMAO-like update rules can be interpreted as bounded closed-loop search systems in which private energy, communal budget, role drift, and population turnover regulate each other endogenously. Under mild assumptions, the key internal states remain bounded, the active population behaves as a regenerative bounded process, and contraction, reinvestment, and search redistribution arise as identifiable operating regimes of the loop.

What has been established is therefore narrower than a complete theory of MMAO performance. The paper does not prove global convergence of the full adaptive algorithm, does not derive exact exploration-exploitation schedules from first principles, and does not show that any bounded metabolic instantiation is automatically competitive on difficult benchmarks. The mechanism-validation package should likewise be read as supportive evidence for the proposed interpretation, not as a definitive causal identification of every internal pathway.

The results support an interpretation of MMAO as a resource-regulated metaheuristic architecture with identifiable internal regulation mechanisms. That identity is still incomplete and should be strengthened by more formal proofs, simplified analyzable variants, stronger mechanism diagnostics, and broader cross-domain validation. The current analysis identifies which regulation properties follow from the model under its assumptions.

The contribution is to distinguish the loop's generic regulation properties from implementation-dependent behavior.

\appendices
\section{Additional Proof Details}
This appendix expands the proof sketches in the main text into a more explicit argument structure. The goal is still not a full textbook-style convergence proof for the complete adaptive system; rather, it is to make the boundedness and regime claims less schematic and more mathematically legible.

\subsection{Projected Affine Bounds}
The first step is to isolate the generic recursion structure used repeatedly in the paper.

\begin{lemma}[Uniform bound for projected affine recursion]
Let $\{y_t\}$ satisfy
\[
y_{t+1}=\Pi_{[0,\bar y]}\big((1-\rho_t)y_t+\eta_t\big),
\]
where $\rho_t\in[\rho_{\min},1)$ for some $\rho_{\min}>0$ and $|\eta_t|\le K$. Then
\[
0\le y_t\le \max\{y_0,\bar y,K/\rho_{\min}\}
\qquad \text{for all } t.
\]
\end{lemma}

\begin{proof}
Because projection preserves the interval $[0,\bar y]$, nonnegativity is immediate. For the upper bound, use
\[
y_{t+1}\le (1-\rho_{\min})y_t+K.
\]
If $y_t\ge K/\rho_{\min}$, then
\[
(1-\rho_{\min})y_t+K \le y_t.
\]
Hence the scalar recursion cannot increase above $\max\{y_0,K/\rho_{\min}\}$ before projection, and projection can only reduce it further or cap it by $\bar y$.
\end{proof}

For Proposition 1, the key observation is that each energy update is a projected affine recursion with bounded additive forcing. Writing
\[
E_i(t+1)=\Pi_E\big((1-\alpha_t)E_i(t)+\zeta_i(t)\big),
\]
where $\zeta_i(t)=u_i(t)-d_i(t)-c_i(t)$, Assumptions 1 and 2 imply
\[
|\zeta_i(t)| \le U_{\max}+D_{\max}+C_{\max}=:K_E.
\]
Applying the lemma with $y_t=E_i(t)$ yields the claimed uniform bound whenever the implementation enforces a positive minimal decay or an explicit projection cap. The same argument applies to the communal budget recursion with
\[
B_{t+1}=\Pi_B(B_t+\eta_t), \qquad |\eta_t|\le N_{\max}D_{\max}+S_{\max}=:K_B.
\]
If the budget update is implemented with explicit bounded projection, compactness is immediate. If instead only nonnegativity projection is used, the dissipativity condition of Theorem 1 supplies the missing upper control.

\subsection{Drift View of Theorem 1}
We now make the communal-budget argument slightly more formal.

\begin{lemma}[Negative drift outside a finite budget region]
Assume the hypotheses of Theorem 1 and let $C=[0,B^\dagger]$. Then for $V(B)=B$,
\[
\mathbb{E}[V(B_{t+1})-V(B_t)\mid B_t]\le -\varepsilon.
\]
\end{lemma}

\begin{proof}
For $B_t\ge B^\dagger$, the hypothesis gives
\[
\mathbb{E}\!\left[s_t-\sum_{i=1}^{N_t} d_i(t)\mid B_t\right]\ge \varepsilon.
\]
Using the update
\[
B_{t+1}=\Pi_B\!\left(B_t+\sum_{i=1}^{N_t}d_i(t)-s_t\right)
\]
and the fact that projection onto the feasible budget interval cannot increase the argument, we obtain
\[
\mathbb{E}[B_{t+1}-B_t\mid B_t]\le -\varepsilon
\]
for all $B_t\notin C$.
\end{proof}

Summing the drift inequality until the hitting time of $C$ gives a standard Foster-Lyapunov estimate: the expected excursion length outside $C$ is finite, and therefore the budget process has bounded first moments and repeatedly re-enters the low-budget region. The result is intentionally stated in expectation rather than almost sure convergence form because the full MMAO loop contains state-dependent admissions, removals, and role couplings that make a stronger claim harder to justify at the framework level.

\subsection{Population Regeneration}
To sharpen Theorem 2, define the refreshed interior set
\begin{equation}
\mathcal{R}=\{(N_t,E_t): N_t\in[N_{\min},N_{\max}]\},
\end{equation}
and call a state in $\mathcal{R}$ \emph{refreshed} if at least one agent has been admitted or reset within the last $L$ steps, for some fixed finite refresh horizon $L$.

\begin{lemma}[Refresh recurrence under bounded turnover policy]
Suppose there exists $p_{\mathrm{rm}}>0$ such that any agent staying below the survival threshold for $L$ consecutive steps is removed with probability at least $p_{\mathrm{rm}}$. Suppose also that admissions occur whenever $N_t<N_{\min}$ or whenever surplus-triggered expansion is activated. Then the process returns to $\mathcal{R}$ infinitely often in expectation.
\end{lemma}

\begin{proof}
The envelope condition gives $N_t\in[N_{\min},N_{\max}]$ for all $t$. Whenever a nonempty stale subset persists below threshold for $L$ steps, at least one removal occurs with probability at least $p_{\mathrm{rm}}$. If this removal pushes the population below the protected floor, admissions are forced; if not, surplus-triggered admissions may still occur under the stated policy. In either case, refresh events have positive probability on each sufficiently long stale interval. Therefore the process cannot remain forever in a stale no-refresh configuration without contradicting the assumed turnover mechanism, and returns to $\mathcal{R}$ recur.
\end{proof}

This is the precise sense in which the population process is regenerative here: not in the strongest renewal-theoretic sense of i.i.d. cycles, but in the weaker and more appropriate sense of repeated returns to a bounded refreshed set.

\subsection{Why the Regime Propositions Need Only One-Sided Drift}
The regime propositions are not asymptotic theorems about permanent phases; they are finite-window structural statements. Their mathematical burden is therefore lighter. For contraction, one needs only a persistent negative drift of private energy for a nontrivial fraction of agents and insufficient admissions to offset removals. For reinvestment, one needs only monotonicity of spending in communal surplus together with nonnegative recent returns. For redistribution, one needs only that future access to resources depends monotonically on recent marginal-return estimates. These are weaker than full equilibrium assumptions, which is appropriate because the goal is to explain observed search modes rather than to prove that the entire system converges to one stationary behavior class.

\section*{Acknowledgement}
The related publications in Table~\ref{tab:mmao-research-series} provide complementary views of the MMAO research series.

\begin{table}[!ht]
\centering
\caption{Other publications in the MMAO research series.}
\label{tab:mmao-research-series}
\footnotesize
\setlength{\tabcolsep}{5pt}
\renewcommand{\arraystretch}{1.12}
\begin{tabular}{@{}p{\columnwidth}@{}}
\hline
\textbf{Related publication} \\
\hline
\emph{MMAO: A Metabolic Multi-Agent Optimizer with Endogenous Resource Allocation for Continuous and Discrete Optimization}~\cite{xu2026mmao}.\newline
\href{https://arxiv.org/abs/2606.28109}{arXiv:2606.28109}
\\[5pt]
\emph{Minimal MMAO: A Resource-Closed-Loop Framework for Adaptive Metaheuristic Search}~\cite{xu2026minimalmmao}.\newline
\href{https://arxiv.org/abs/2606.30450}{arXiv:2606.30450}
\\[5pt]
\emph{A Large-Scale Empirical Evaluation of MMAO Under Fair-Budget Continuous and Discrete Benchmarks}~\cite{xu2026largescale}.\newline
\href{https://arxiv.org/abs/2606.31584}{arXiv:2606.31584}
\\[5pt]
\emph{MMAO-Dyn: A Metabolic Multi-Agent Optimizer for Dynamic Optimization}~\cite{xu2026mmaodyn}.\newline
\href{https://arxiv.org/abs/2607.00846}{arXiv:2607.00846}
\\[5pt]
\emph{MMAO-Cls: Metabolic Multi-Agent Optimization for Joint Feature Selection and Classifier Tuning}~\cite{xu2026mmaocls}.\newline
\href{https://arxiv.org/abs/2607.01539}{arXiv:2607.01539}
\\[5pt]
\emph{MMAO: A Unified Resource-Regulated Controller for Cross-Domain Optimization}~\cite{xu2026unifiedmmao}.\newline
\href{https://doi.org/10.1109/ACCESS.2026.3752083}{IEEE Access, 2026}
\\[5pt]
\hline
\end{tabular}
\end{table}

\bibliographystyle{IEEEtran}
\bibliography{Ref}

\end{document}